\documentclass[letterpaper, 10 pt, conference]{ieeeconf}  

\IEEEoverridecommandlockouts                              

\usepackage{cite}
\usepackage{color}
\usepackage{latexsym}
\usepackage{amsmath}
\usepackage{amssymb}
\usepackage{algorithm}
\usepackage{algorithmic}
\usepackage{graphics} 
\usepackage{epsfig} 
\usepackage{booktabs} 
\usepackage{times}
\usepackage{float}
\usepackage{algorithm}
\usepackage{makecell}
\usepackage{multirow}
\usepackage{hhline}
\usepackage{algorithmic}
\usepackage{mathrsfs}
\usepackage{wrapfig}

\newtheorem{myTheo}{Theorem}

\newtheorem{lemma}{Lemma}

\newtheorem{mypro}[myTheo]{Proposition}

\newtheorem{assumption}{Assumption}

\title{\LARGE \bf Taming Convergence for Asynchronous Stochastic Gradient Descent with Unbounded Delay in Non-Convex Learning
}

\author{Xin Zhang, Jia Liu and Zhengyuan Zhu
\thanks{
Xin Zhang is with the Department of Statistics and Department of Computer Science, Iowa State University,
        Ames, IA 50011, USA
        {\tt\small xinzhang@iastate.edu}}
\thanks{
Zhengyuan Zhu is with the Department of Statistics, Iowa State University,
        Ames, IA 50011, USA
        {\tt\small zhuz@iastate.edu}}
\thanks{Jia Liu is with the Department of  Computer Science and Department of Electrical and Computer Engineering, Iowa State University, Ames, IA 50011, USA
        {\tt\small jialiu@iastate.edu}}%
}

\begin{document}


\maketitle
\thispagestyle{empty}
\pagestyle{empty}

\begin{abstract}
Understanding the convergence performance of asynchronous stochastic gradient descent method (Async-SGD) has received increasing attention in recent years due to their foundational role in machine learning.
To date, however, most of the existing works are restricted to either  bounded gradient delays or convex settings.
In this paper, we focus on Async-SGD and its variant Async-SGDI (which uses increasing batch size) for non-convex optimization problems with unbounded gradient delays. 
We prove $o(1/\sqrt{k})$ convergence rate for Async-SGD and $o(1/k)$ for Async-SGDI.
Also, a unifying sufficient assumption for Async-SGD's convergence is proposed, which includes two major gradient  delay models in the literature as special cases. 
\end{abstract}

\section{Introduction} \label{sec:intro}

Fueled by large-scale machine learning and data analytics, recent years have witnessed an ever-increasing need for computing power.
However, with the miniaturization of transistors nearing the limit at atomic scale, it is projected that the celebrated Moore's law (the doubling growth rate  of CPU speed in every 18 months) will end in around 2025 \cite{MooresLawEnd}.
Consequently, to sustain the rapid growth for machine learning technologies in the post-Moore's-Law era, the only viable solution is to exploit {\em parallelism} at and across different spatial scales.
Indeed, the recent success of machine learning research and applications is due in a large part to the advances in multi-core CPU/GPU technologies (on the micro chip level) and networked cloud computing (on the macro data center level), which enable the developments of highly parallel and distributed algorithmic architectures.
Such examples include parallel SVM \cite{zhang2005parallel}, scalable matrix factorization \cite{yu2012scalable,yu2014parallel,xu2013parallel}, distributed deep learning \cite{dean2012large,povey2014parallel,abadi2016tensorflow,li2014scaling}, to name just a few.


However, developing efficient and effective parallel algorithms is highly non-trivial. 
In the literature, most parallel machine learning algorithms are synchronous in nature, i.e., a set of processors performing certain computational tasks in a distributed fashion under a common clock.
Although synchronous parallel algorithms are relatively simpler to design and analyze theoretically, their implementations in practice are usually problematic:
First, in many computing systems, maintaining clock synchronization is expensive and incurs high complexity and system overhead.
Second, synchronous parallel algorithms do not work well under heterogenous computing environments since all processors must wait for the slowest processor to finish in each iteration. 
Exacerbating the problem is the fact that, in many machine learning applications, it is often difficult to decompose a problem into subproblems with similar difficulty. This introduces yet another layer of heterogeneity in CPU/GPU processing time.
Third, synchronous operations in parallel algorithms often induce periodic spikes in information exchanges and congestions in the systems, which further cause high communication latency and even information losses. 
Due to these limitations, it is not only desirable but also necessary to consider {\em asynchronous parallel algorithmic designs} in practice.


In an asynchronous parallel algorithm, rather than making updates simultaneously, each node computes its own solution in each iteration without waiting for other nodes in the system.
Compared with their synchronous counterparts, asynchronous parallel algorithms are more resilient to heterogeneous computing environments and cause less network congestions and delay.
As a result, asynchronous parallel algorithmic designs are more attractive in practice for solving large-scale machine learning problems.
However, one of the most critical issues of asynchronous parallel  algorithms is that the use of stale system state information is unavoidable due to the asynchronous updates.
If not treated carefully, such delayed system information could destroy the convergence performance of their synchronous versions. 
This problem is particularly concerning for the {\em asynchronous stochastic gradient descent method} (Async-SGD), which is the fundamental building block of many distributed machine learning frameworks in use today (e.g., TensorFlow, MXNet, Caffe, etc.).
Hence, understanding the convergence performance of Async-SGD is important and has received increasing attention in recent years (see, e.g., \cite{recht2011hogwild,lian2015asynchronous,zheng2017asynchronous,huo2017asynchronous} etc.).
To date, however, most of the existing work in this area are restricted to the bounded gradient delay setting, whereas results on {\em unbounded} gradient delay remain quite limited (see Section~\ref{sec:related} for more detailed discussions). 
Moreover, all existing convergence results with unbounded delay in the literature require convexity assumptions, which are irrelevant to the inherently non-convex nature of many challenging machine learning problems.
In light of these limitations, our goal is in this paper is to fill this gap and achieve a deeper understanding of the convergence performance of Async-SGD in non-convex learning.

Toward this end, we consider using Async-SGD to solve a non-convex optimization problem in the form of:
\begin{equation}\label{ob.prob}
\min_{x\in \mathbb{R}^d} f(x)=\mathbb{E}[F(x;\xi)],
\end{equation}
where $\xi$ is an i.i.d. random sample drawn from the database, and $f(x)$ is a smooth non-convex function. 
The objective in (\ref{ob.prob}) could be infinite-sum, which means the sample size in database is large.
We note that Problem (\ref{ob.prob}) is general enough to represent a wide range of machine learning problems in practice.
Further, we do {\em not} assume any bounded delay of the outdated stochastic gradients during the execution of Async-SGD.
As will be shown later, the unbounded assumption significantly complicates the convergence analysis of Async-SGD.
Our main technical results and key contributions in this paper are summarized as follows:

\vspace{-.01in}
\begin{list}{\labelitemi}{\leftmargin=1em \itemindent=-0.09em \itemsep=.2em}
\item First, we show that by choosing step-sizes at the speed $O(1/(\sqrt{k}\log(k)))$, $\mathbb{E}\{ \| \nabla f(x_k) \|_{2} \}$ converges to zero with rate $o(1/\sqrt{k}),$
which is much stronger compared to the $O(1/\sqrt{k})$ convergence rate in existing works of this area (see, e.g. \cite{lian2015asynchronous} and references therein). 
This is a surprising result because, to our knowledge, most existing work in the literature only yields Big-O bounds (e.g., $O(1/\sqrt{k})$).
In other words, our result shows that unbounded gradient delay in Async-SGD actually makes {\em no} difference in terms of convergence rate (in order sense) compared to the synchronous version.
This finding generalizes the existing results in the literature.

\item Second, by leveraging a supermartingale convergence theorem, we propose a generalized and more relaxed sufficient assumption on the probability distribution of the gradient-updating delay that guarantees convergence.
Our sufficient assumption offers a {\em unifying} framework that includes  two major gradient update delay models often assumed in the literature as special cases, namely: 1) bounded delays, and 2) unbounded i.i.d. delay (see more detailed discussions in Section~\ref{sec:related}).
Further, our sufficient assumption includes the delay distributions across iterations which could be non-i.i.d, unbounded, and even {\em heavy-tailed} (e.g., log-normal, Weibull, etc.). 

\item 
Inspired by the trade-off between ``rate'' and ``variance'',
we consider a variant of Async-SGD with increasing batch size (Async-SGDI).
We show that, if the batch size grows at rate $\omega(k)$, Async-SGDI achieves an $o(1/k)$ convergence rate result under a {\em fixed} step-size.
In other words, as long as the batch size grows slightly faster than linear, a  small constant step-size is sufficient to achieve an even faster Small-O convergence rate.
Therefore, there is no need to be concerned with the use of vanishing step-size strategies, which could be problematic because of numerical instability in practice.
\end{list}

The rest of the paper is organized as follows. Related work is discussed in Section 2. We will present the system model of Async-SGD in Section 3. In Section 4, the convergence rate of Async-SGD with unbounded delay is derived. The conclusion is given in Section 5. Due to limited space, experiment results and proofs are shown in Supplementary.

\section{Related work} \label{sec:related}
To put our work in comparison perspectives, in this section, we first provide a quick overview on stochastic gradient descent method (SGD).
We then focus on the recent advancements of Async-SGD.

{\bf 1) SGD and variance reduction:}
The SGD algorithm traces its root to the seminal work by \cite{robbins1951stochastic} and \cite{kiefer1952stochastic}, and has become a key component for solving many large-scale optimization problems.
Due to its foundational importance, the convergence rates of SGD and its variants have been actively researched over the years.
It is well known that the convergence rate of SGD is $O(1/\sqrt{k})$ for convex problems (see, e.g., \cite{nemirovski2009robust}) and $O(1/k)$ for strongly convex problems (see, e.g.,  \cite{moulines2011non}). 
To improve the convergence speed of SGD, stochastic variance reduction  methods have been proposed. 
For example, the stochastic averaged gradient (SAG) method proposed in \cite{schmidt2017minimizing} converges at $O(1/k)$ speed for convex problems and converges linearly for strongly convex problems. 
The stochastic variance reduced gradient (SVRG) method proposed in \cite{johnson2013accelerating} also enjoys similar sublinear convergence rate  for convex problems and linear convergence rate for strongly convex problems.

{\bf 2) SGD for non-convex problems:}
Due to the inherent non-convex nature in training deep neural networks, the convergence performance of SGD for non-convex optimization problems has also become a focal research area recently.
For example, \cite{ghadimi2013stochastic} proved that the ergodic convergence rate for nonconvex objection function with $\sigma-$bounded gradient is $O(1/\sqrt{k})$. 
Later, \cite{reddi2016stochastic} extended SVRG to non-convex problem and proved that it has a sublinear convergence. 
We note that this convergence rate result is consistent with that of the convex case.

{\bf 3) Asynchronous SGD for convex problems:}
As mentioned in Section~\ref{sec:intro}, Async-SGD has become increasingly popular recently due to its simplicity in implementation and practical relevance in many machine learning frameworks.
One of the earliest studies on Async-SGD is the algorithm termed HOGWILD! in \cite{recht2011hogwild}. 
HOGWILD! is a lock-free asynchronous parallel implementation of SGD on the shared memory system with sublinear convergence rate for strongly convex smooth problems. 
At roughly the same time, \cite{agarwal2011distributed} studied the convergence performance of SGD-based optimization algorithms on distributed stochastic convex problems with asynchronous and yet delayed gradients. 
Interestingly, asymptotic convergence rate $O(1/\sqrt{k})$ is shown in their work, which is consistent with that of the non-delayed case. 
However, compactness of feasible domain and bounded gradient are assumed in this work.
In \cite{reddi2015variance}, asynchronous stochastic variance reduction (Async-SVR) methods were analyzed for convex objectives and bounded delay. 
Note that all aforementioned Async-SGD methods assumed {\em bounded} gradient delay. 
One of the first investigations on unbounded delay is due to \cite{hannah2016unbounded}, where the convergence rate of ARock, an asynchronous coordinate decent method for solving convex optimization problems, is considered. 
It was shown that ARock converges weakly to a solution with probability one if the unbounded delayed gradients are independent and identically distributed (i.i.d.). 

\begin{table*}[t!]
\caption{Convergence comparisons for existing asynchronous methods ($\rho \in (0,1)$ is a constant; "Sum" means whether the total size of sample is finite or not).}
\begin{center}
\begin{tabular}{ c |c|c| c | c |c }

\hline
Work & Method & Sum & Convexity &Delay & Rate \\

\hline
\hline
\cite{hannah2016unbounded} & ARock & - & convex  & unbounded & - \\
\hline
\cite{sra2015adadelay} & Adadelay & infinite-sum & convex & unbounded & $O(1/k)$\\
\hline
\multirow{3}{*}{\cite{sun2017asynchronous} }& \multirow{3}{*}{Async-BCD} & \multirow{3}{*}{-} & strongly convex & \multirow{3}{*}{bounded} & $O(\rho^k)$ \\
\hhline{~~~-~-} &&&convex & & $o(1/k)$\\
\hhline{~~~-~-} &&&nonconvex &  & $o(1/\sqrt{k})$\\
\hline
\cite{lian2015asynchronous}&Async-SGD& finite-sum&nonconvex & bounded& $O(1/\sqrt{k})$ \\
\hline
\cite{huo2017asynchronous}&Async-SVRG& finite-sum&nonconvex& bounded&$O(1/k)$ \\
\hline
\multirow{2}{*}{Our work} &Async-SGD & \multirow{2}{*}{infinite-sum} & nonconvex  &unbounded & $o(1/\sqrt{k})$ \\ 
 \hhline{~-~---}        & Async-SGDI & &nonconvex &unbounded&$o(1/k)$ \\ 
\hline
\end{tabular}
\label{tab:comp}
\end{center}
\vspace{-.1in}
\end{table*} 

\smallskip
{\bf 4) Asynchronous SGD for non-convex problems:}
Similar to their synchronous counterparts, Async-SGD for nonconvex optimization problems also starts to attract some attentions lately. 
For example, Ref \cite{lian2015asynchronous} studied the convergence rate of Async-SGD for non-convex optimization problems with bounded delay, where they showed an $O(1/\sqrt{k})$ sublinear convergence rate. 
However, the best convergence rate they provided is highly dependent on the step-size selection strategy, which in turn depends on some {\em a priori} iteration threshold value $K$. 
Most recently in \cite{huo2017asynchronous}, an asynchronous mini-batch SVRG with bounded delay is proposed for solving non-convex optimization problems. 
They proved that the proposed method converges with an $O(1/k)$ convergence rate for non-convex optimization.

To conclude this section, we summarize the convergence performance guarantees in the prior literature and our results in Table~\ref{tab:comp} for clearer comparisons.

\section{System model and the asynchronous stochastic gradient descent algorithms} \label{sec:model}

\begin{figure}
\centering
\includegraphics[width=0.6\linewidth]{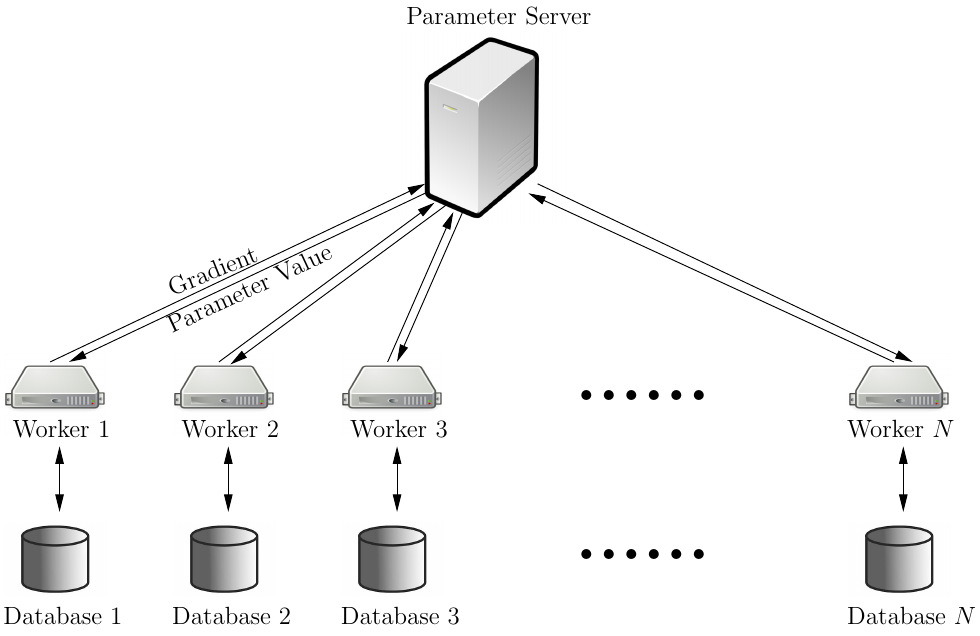}
\vspace{-.1in}
\caption{A parallel computing architecture for asynchronous gradient descent (Async-SGD).}
\label{fig:SysArch}
\vspace{-.1in}
\end{figure}

In this section, we first describe the system model for Asynchronous parallel algorithms. 
Then, we present the standard Async-SGD algorithm and a variant of Async-SGD with increasing batch size, which is named Asyn-SGDI.

Consider solving the optimization problem in (1) in a parallel computing architecture consisting of a parameter server and $N$ workers ($N$ is oftenly a fixed number), as shown in Figure~\ref{fig:SysArch}.
In practice, each worker could be a GPU (on a chip-scale) or a standalone server (on a datacenter-scale).
Under the Async-SGD algorithm, each worker independently retrieves the current values $x_{k}$ from the parameter server and randomly select a mini-batch of data samples from database and compute the stochastic gradient. 
Once the computation is finished, each worker immediately reports the computed stochastic gradient to the parameter server without waiting for other workers, and then start next computing cycle. 
On the other hand, upon collecting $M$ gradients from workers, the parameter server updates its current parameter with these stochastic gradients. 
However, due to asynchronicity, the server could use stale gradient information to update the parameters, which will affect the convergence.
We present Async-SGD in Algorithm~1.

\medskip
\hrule 
\vspace{0.01in}
\noindent {\bf Algorithm~1: Asynchronous SGD (Async-SGD).}
\vspace{.02in}
\hrule
\noindent {\bf At the parameter server:}
\begin{enumerate}
\item[1.] In the $i$-th update, wait till collecting $M$ stochastic gradients $G(x_{i-\tau_{i,m}};\xi_{i,m})$ from the workers. 
\item[2.] Update: $x_{i+1}=x_i-\gamma_i\sum_{m=1}^{M}G(x_{i-\tau_{i,m}};\xi_{i,m})$. 
\end{enumerate}

\noindent {\bf At each worker:}
\begin{enumerate}
\item[1.] Retrieve the current value of parameter $x$ from the parameter server. 
\item[2.] Randomly select a sample $\xi$ from the database. 
\item[3.] Compute stochastic gradient $G(x;\xi)$ and report it to server.
\end{enumerate}
\smallskip
\hrule

\medskip

%

\medskip
\hrule 
\vspace{.01in}
\noindent {\bf Algorithm~2: Async-SGD with increasing batch.} 
\vspace{.02in}
\hrule
\noindent {\bf At the parameter server:}
\begin{enumerate}
\item[1.] In the $i$-th update, wait till collecting $n_iM$ stochastic gradients $G(x_{i-\tau_{i,m}};\xi_{i,m})$ from workers. 
\item[2.] Update:$x_{i+1}=x_i - \frac{\gamma_i}{n_{i}} \sum_{m=1}^{n_iM}G(x_{i-\tau_{i,m}};\xi_{i,m})$. 
\end{enumerate}

\noindent {\bf At each worker:}
\begin{enumerate} 
\item[1.] Retrieve the current value of parameter $x$ from the parameter server. 
\item[2.] Randomly select a sample $\xi$ from the database. 
\item[3.] Compute stochastic gradient $G(x;\xi)$ and report it to server.
\end{enumerate}
\smallskip
\hrule

\medskip

In Algorithm~1, $G(x;\xi)$ denotes a stochastic gradient of $f(x)$ that is dependent on a random sample $\xi$; 
$\tau_{i,m}$ represents the delay for the $m$-th gradient in the mini-batch in $i$-th update seen by the parameter server. 
As shown in Algorithm~1, the parameter server updates the parameters regardless of the freshness of the collected gradients.

Further, 
we consider a modified scheme for Async-SGD with the same system. 
Instead of a fixed number of gradients, the server collects an increasing number of gradients as the number of iterations increases to help reduce the variance of stochastic gradients. 
We outline this scheme in Algorithm~2. 
Apparently, compared to the basic Async-SGD, the only difference between the two algorithms is that the batch size $n_{i}M$ at the parameter server is increasing, where $\{n_{i}\}_{i=1}^{\infty}$ is an integer-valued increasing series.

\section{Convergence analysis} \label{sec:main_thms}

In this section, we will conduct convergence analysis for the two Async-SGD algorithms described in Section 3. 
Similar to previous work on optimization for non-convex learning problems (see, e.g., \cite{cartis2010complexity,gratton2008recursive}), we use the expected $\ell_2$ norm of the gradient, i.e., $\mathbb{E} \{ \|\nabla f(x)\|^2 \}$, as the convergence metric. 
For non-convex optimization problems, we show that Async-SGD converges to a stationary point with asymptotic convergence rate $o(1/\sqrt{k})$. 
For Async-SGDI, the asymptotic convergence rate is even faster at $o(1/k)$. 

\subsection{Assumptions}

We first state the following assumptions for our analysis. 
The first three are commonly assumed in the literature for analyzing the convergence of SGD. 
The fourth assumption is a sufficient condition for the characteristics of gradient delays under which the convergence of Async-SGD is guaranteed.

\begin{assumption}[Lower bounded objective function] \label{assump_bnded_obj}
For the objective function $f$, there exists an optimal solution $x^*$, such that $\forall x \ne x^*$, we have $f(x)\ge f(x^*).$
\end{assumption}

\begin{assumption}[Lipschitz continuous gradient] \label{assump_Lipschitz_grad}
There exists a constant $L>0$ such that the objective function $f(\cdot)$ satisfies $\|\nabla f(x)-\nabla f(y)\|\le L\|x-y\|, ~\forall x,y \in \mathbb{R}^{d}.$
\end{assumption}

\begin{assumption}[Unbiased gradients with bounded variance]\label{assump_grad_exp_var}  The stochastic gradient $G(x;\xi)$ satisfies: $\mathbb{E}(G(x;\xi))=\nabla f(x)$, $\forall x,\xi$, and $\mathbb{E}(\|(G(x;\xi))-\nabla f(x)\|^2)\le \sigma^2$, $\forall x$.
\end{assumption}

\begin{assumption}[Uniformly Upper Bounded Delay]\label{assump_delay} 
Consider the probability series of random delays $\{\tau_k\}_{k=1}^{\infty}$. 
There exists a series $\{a_i\}_{i=1}^{\infty}$ such that i) $\mathbb{P}(\tau_k=i) \le a_i$, $\forall k$, and ii) $\sum_{i=1}^{\infty} i^2a_i < \infty.$
\end{assumption}

Assumption~\ref{assump_delay} covers the delays that are {\em heavy-tailed distributed}.
This includes discrete log-normal, discrete T-distribution, discrete Weibull, etc. Here, discrete log-normal means $\mathbb{P}(\tau=i)=\mathbb{P}(i\le x <i+1),$ where $x$ is a random variable that is log-normal distributed. 
Similar distributions can be defined for discrete T-distribution and discrete Weibull
We will further discuss the extension of this assumption in Section~\ref{subsec_delay_var}.

\subsection{Convergence for Async-SGD with Unbounded Delay}

To establish the convergence results of Async-SGD with unbounded delay, consider the following Lyapunov funct{}ion:
\begin{equation} \label{eqn_Lyapunov}
\zeta^k=f(x_k)-f(x^*)+\sum_{j=1}^{k}c_j\|x_{k+1-j}-x_{k-j}\|^2.
\end{equation}
In the Lyapunov function in (\ref{eqn_Lyapunov}), the first part $f(x_{k})-f(x^{*})$ measures the optimality error between current objective value and the optimal objective value.
The second part $\sum_{j=1}^{\infty}c_j\|x_{k+1-j}-x_{k-j}\|^2$ in (\ref{eqn_Lyapunov}) is a weighted sum of the distances between past iterates, where the properties of weights $\{c_i\}_{i=1}^{\infty}$ will be described soon.
The second part can be viewed as the accumulative error due to asynchronous gradient updates. 
Here, we give the following lemma to establish the relationship between the delay probability series $\{a_i\}_{i=1}^{\infty}$ and the weight sequence $\{c_i\}_{i=1}^{\infty}.$ 
\begin{lemma}\label{lemma: delay probability} 
Under Assumption \ref{assump_delay}, there exists a non-negative sequence $\{c_i\}_{i=1}^{\infty}$, such that 
\begin{equation}\label{relation}
c_{j+1}+\frac{\gamma_kML^2}{2}\sum_{i=j}^{k}i\mathbb{P}(\tau_{k}=i)\le c_j,~\forall~ k,
\end{equation}  where $\tau_k$ denotes the maximum delay in $k$-th iteration, i.e., $\tau_k=\max_{m} \tau_{k,m}$ and $\gamma_k$ is the step-size.
\end{lemma}
The proof of Lemma \ref{lemma: delay probability}  can be found in Appendix \ref{appdx_lem1}.
Lemma \ref{lemma: delay probability} shows that under the delay probability series in Assumption \ref{assump_delay}, the weight sequence will be non-negative, which gaurantees the existence of the Lyapunov function (\ref{eqn_Lyapunov}).
We note that a similar type of Lyapunov function was used in \cite{hannah2016unbounded}, where they used $\|x_k-x^*\|^2$ as the optimality error thanks to the non-expansiveness assumption therein. 
Similar to the discussions in previous work, because of asynchronicity, it is hard to directly show the contraction relationship $\mathbb{E}[f(x_k)-f(x^*)]\le f(x_{k-1})-f(x^*).$ 
However, we can prove the following inequality for the proposed Lyapunov function $\zeta^{k}$, which will play a key role in our subsequent analysis.

\begin{lemma}\label{lemma1}
Under Assumptions~\ref{assump_bnded_obj}--\ref{assump_delay}, if the step-size $\{\gamma_k\}_{k=1}^{\infty}$ satisfies that $\gamma_k\le 1/(2Mc_1+M)$, $\forall k$, then the following inequality holds:
\begin{align} \label{eqn_lyapunov_ineq}
\mathbb{E}\{\zeta^{k+1}|\mathscr{F}^k\}+\frac{\gamma_kM}{2}&\|\nabla f(x_k)\|^2 \notag\\
&
\le \zeta^{k}+ (c_1\gamma_k^2M+\frac{L\gamma_k^2M}{2})\sigma^2.
\end{align}
where $\mathscr{F}^k$ represents the filtration of the history of iterates and delays, i.e., $\mathscr{F}^k = \sigma \langle x_0,x_1,\ldots,x_k;$ $\tau_{1},\ldots,\tau_{k}\rangle$
\end{lemma}
The proof of Lemma \ref{lemma1} can be found in Appendix \ref{appdx_thm1}.
Lemma~\ref{lemma1} connects the total error $\zeta$ and the convergence criterion $\| \nabla f(\cdot) \|^2$. 
Intuitively, we can see that if the second term in the right hand side of (\ref{eqn_lyapunov_ineq}) is summable, then $\|\nabla f(\cdot) \|^2$ should also be summable. 
Based on Lemma~\ref{lemma1} and by applying the supermartingale convergence theorem in \cite{hannah2016unbounded,combettes2015stochastic}, we have following main convergence result for Async-SGD:

\begin{myTheo} \label{Theo1}
Under Assumptions~\ref{assump_bnded_obj}--\ref{assump_delay}, if the step-size sequence $\{\gamma_k\}_{k=1}^{\infty}$ satisfies: i) $\gamma_k\le 1/(2Mc_1+ML)$, $\forall k$; ii) $\sum_{k=1}^{\infty} \gamma_{k} = \infty$; and iii) $\sum_{k=1}^{\infty} \gamma_{k}^{2} < \infty$, where $M$ is the fixed batch size, $L$ is the Lipschitz constant in Assumption 2, and $c_1$ is the first element in the sequence $\{c_{i}\}_{i=1}^{\infty}$,
then we have $\mathbb{E}\{\sum_{k=1}^{\infty}\gamma_k\|\nabla f(x_k)\|^2\}< \infty$ and $\mathbb{E}\{\|\nabla f(x_k)\|^2\} \rightarrow 0.$
\end{myTheo}

Due to space limitation, we relegate the proof of Theorem~\ref{Theo1} to Appendix \ref{appdx_thm1}.
Next, we show that Theorem~\ref{Theo1} implies that we can properly choose the step-size sequence $\{\gamma_k\}_{k=1}^{\infty}$ to obtain an $o(1/\sqrt{k})$ convergence rate for Async-SGD:
\begin{mypro} \label{prop_async_sgd}
Consider the diminishing step-size sequence $\gamma_k=O\big(1/k^{1/2}\log(k)\big)$ and $\gamma_k\le 1/(2Mc_1+ML)$, $k=1,2,\ldots$.
Then, the asymptotic convergence rate for Async-SGD is: 
\begin{align} \label{eqn_small_o1}
\mathbb{E}\{\|\nabla f(x_k)\|^2\}=o(1/\sqrt{k}).
\end{align}
\end{mypro}
The basic proof idea of the Little-O rates is based on contradiction, and the detials are provided in Appendix \ref{Append: proposition2}.
%
Proposition~\ref{prop_async_sgd} shows that as the number of iterations increases, the negative effect of outdated gradient information in Async-SGD vanishes asymptotically under the chosen step-sizes. 

\subsection{Async-SGD with increasing batch size}

To analyze convergence performance of Async-SGD with increasing batch size (Async-SGDI), we extend Lemma \ref{lemma1} to obtain following inequality:

\begin{lemma} \label{lemma2}
Under Assumptions~\ref{assump_bnded_obj}--\ref{assump_delay}, if the step-size sequence $\{\gamma_k\}_{k=1}^{\infty}$ satisfies $\gamma_k\le 1/(2Mc_1+ML)$, $\forall k$, then the following inequality holds for Async-SGDI:
\begin{align}
\mathbb{E}\{\zeta^{k+1}|\mathscr{F}^k\}+\frac{\gamma_kM}{2}&\|\nabla f(x_k)\|^2\le \notag\\
& \zeta^{k}+ \Big(c_1\gamma_k^2M+\frac{L\gamma_k^2M}{2} \Big)\frac{\sigma^2}{n_k},
\end{align} 
where $M$ denotes the initial batch size and $\{n_k\}$ is some integer-valued increasing sequence.
\end{lemma}

Then, by applying the supermartingale convergence theorem in a similar fashion, we have the key convergence result for Aysnc-SGDI:

\begin{myTheo}\label{Theo2}
Under Assumptions~\ref{assump_bnded_obj}--\ref{assump_delay}, let the batch size sequence be chosen as $\{M_k:=n_kM\}$, where $M$ is the initial batch size and the integer-valued sequence $\{n_{k}\}_{k=1}^{\infty}$ is increasing and satisfies $\sum_{k=1}^{\infty}1/n_k<\infty$. 
Also, suppose that the step-size $\{\gamma_k\}_{k=1}^{\infty}$ satisfies $\gamma_k\le 1/(2Mc_1+ML)$, $\forall k$. 
Then we have $\mathbb{E}\{\sum_{k=1}^{\infty}\gamma_k\|\nabla f(x_k)\|^2\}< \infty$ and $\mathbb{E}\{\|\nabla f(x_k)\|^2\} \rightarrow 0$.
\end{myTheo}

Again, due to space limitation, we relegate the proof details of Theorem~\ref{Theo2} to our online technique report \cite{zhang2018taming}.
With Theorem~\ref{Theo2}, we claim the following asymptotic convergence rate for Async-SGDI:
\begin{mypro} \label{prop_2}
Let the sequence $\{n_{k}\}_{k=1}^{\infty}$ be chosen as $n_k=\omega(k)$. 
Then, with a fixed step-size satisfying $\gamma \le 1/(2Mc_1+ML)$, we have $\mathbb{E}\{\|\nabla f(x_k)\|^2\}=o(1/k).$
\end{mypro}

Proposition~\ref{prop_2} implies that, by using an increasing batch size, Async-SGDI with a constant step-size converges at rate $o(1/k)$.
However, as batch size increases, the runtime per iteration would also become longer.
Hence, it is unclear whether the number of iteration is a good convergence performance metric.
To resolve this ambiguity, we first note that we only need the batch size to grow at a rate $\omega(k)$.
According to the small-omega definition, the batch size can grow linearly with an arbitrarily small slope, i.e., the batch size increases very slowly.
Second, with constant step-size, the algorithm is much more stable numerically. 

\subsection{Discussion} \label{subsec_delay_var}

We can see from the proofs of Theorems~\ref{Theo1} and \ref{Theo2} that the non-negative sequence $\{c_{i}\}_{i=1}^{\infty}$, gauranteed by Assumption \ref{assump_delay}, is a mathematical construct that plays a key role in establishing the convergence of both Async-SGD algorithms.
In what follows, we show that Assumption~\ref{assump_delay} unifies two known delay models as special cases.
The first special case is the bounded delay model, which has been widely assumed and investigated in the literature (cf. \cite{lian2015asynchronous,huo2017asynchronous} etc.). 
This model is reasonable as long as the gradient computation workload is finite for all workers. 
Assume that $\tau_k$ is bounded by a constant $T$. 
\begin{mypro}\label{cor1}(\textbf{Bounded Delay})
If the random delays in gradient updates $\{\tau_k\}_{k=1}^{\infty}$ are uniformly upper bounded by a constant $T>0$, 
then it satisfies Assumption \ref{assump_delay}.
\end{mypro}

The second delay model is such that the sequence of random delays $\{\tau_k\}_{k=1}^{\infty}$ are i.i.d. and the underlying distribution has a finite second moment. 
This model is reasonable when the number of iterations is large and the system has reached the stationary state. 
For this case, we have the following result:

\begin{mypro}\label{cor2}(\textbf{I.I.D. Random Delay})
If the random delays $\tau_k$, $\forall k$, are i.i.d. and the underlying distribution  has a finite second moment, 
then it satisfies Assumption \ref{assump_delay}.
\end{mypro}

\section{Application examples} \label{sec:numerical}

\begin{figure*}[ht]
\centering
\begin{tabular}{@{}ccc@{}}
\includegraphics[width=5cm]{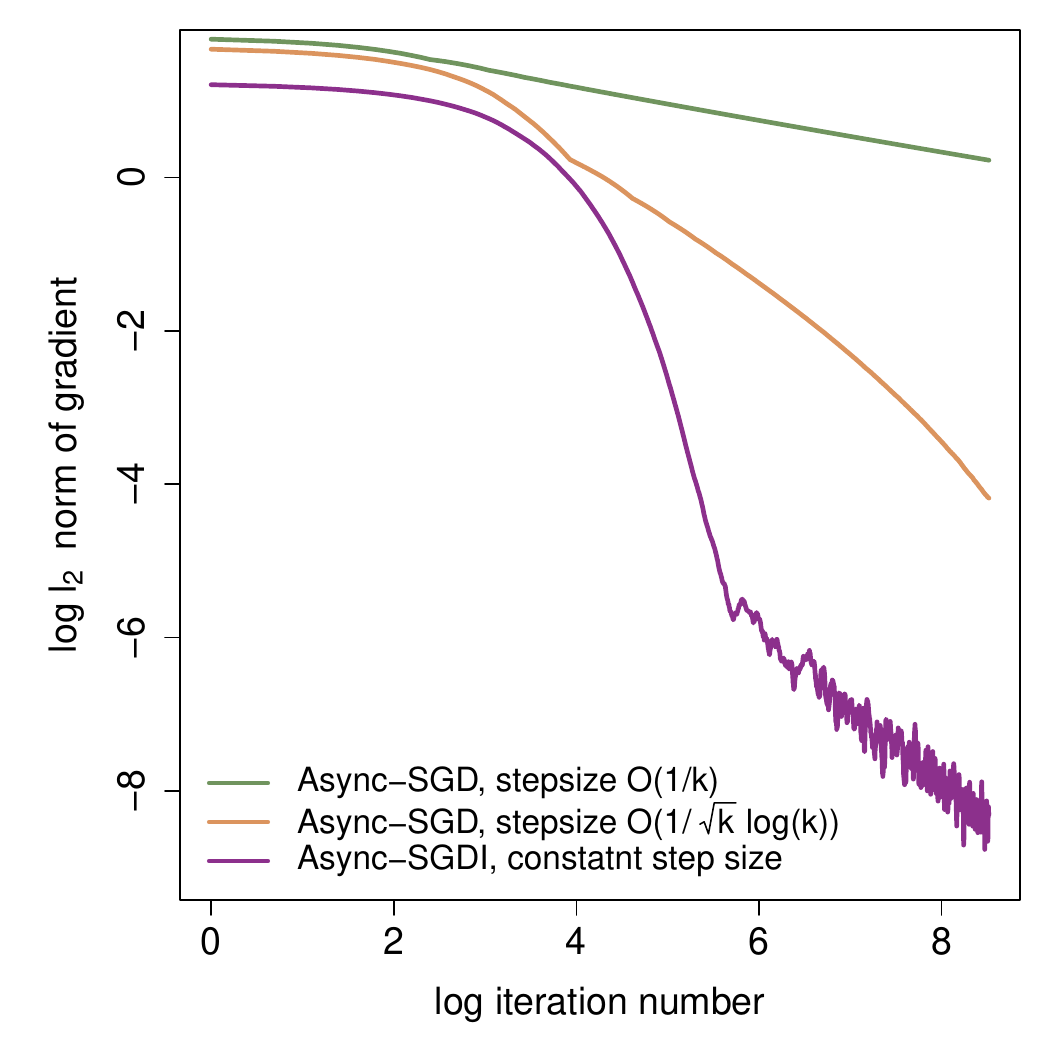}&
\includegraphics[width=5cm]{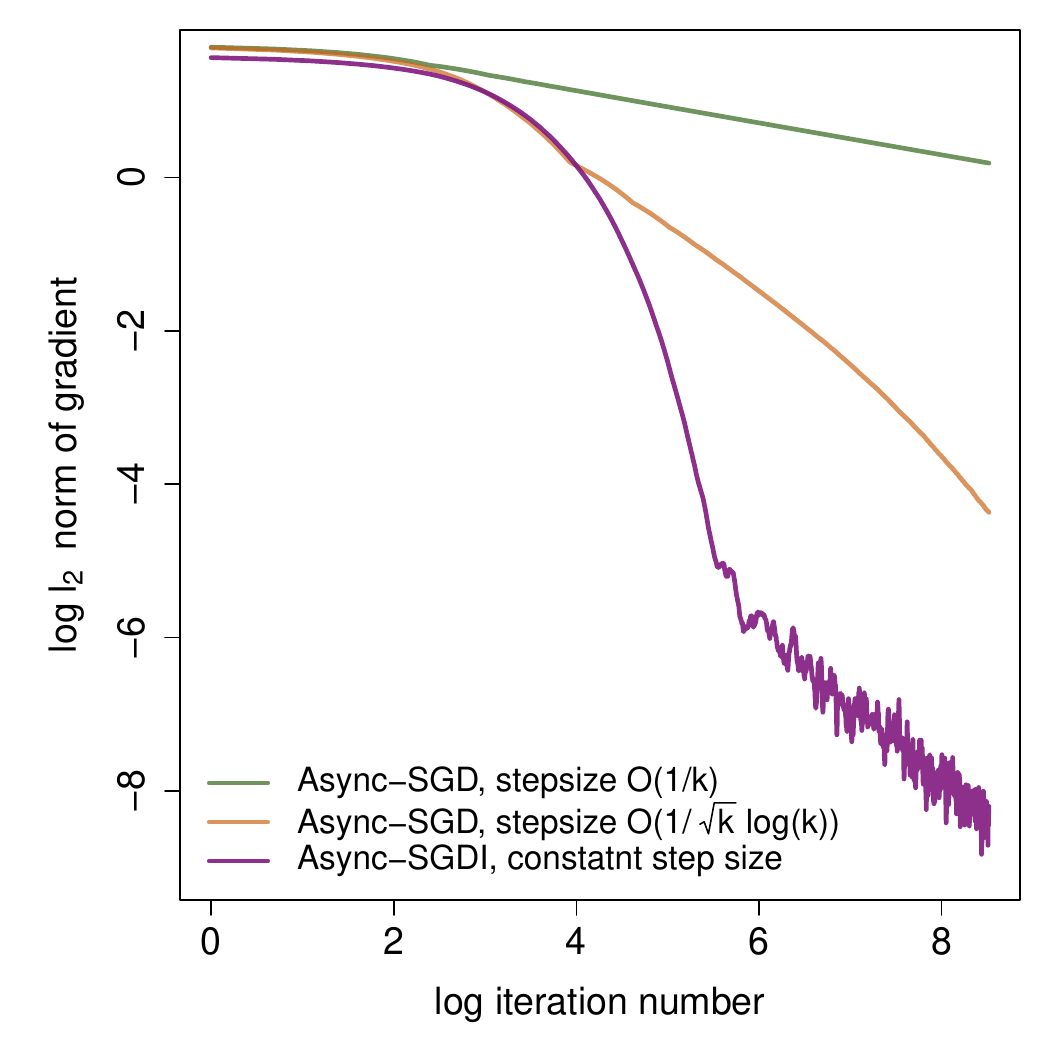}&
\includegraphics[width=5cm]{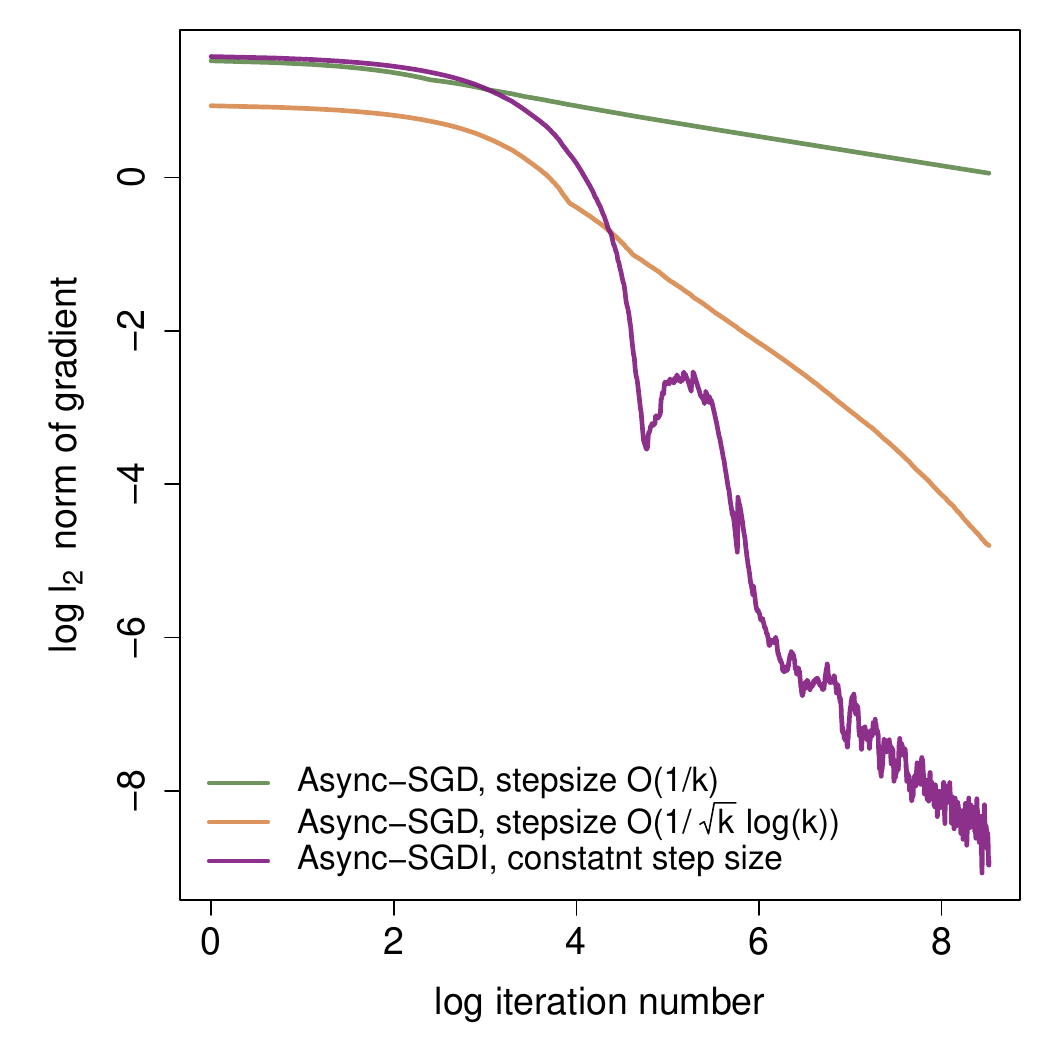}\\
\footnotesize (a) Bounded delay, $\tau\le 20$.&\footnotesize (b) Poisson delay,$\tau\sim Poi(10)$.& \footnotesize (c) System delay.
\end{tabular}
\caption{Simulation results for the convergence of Async-SGD and Async-SGDI with three kinds of delay on matrix completion problem.} 
\vspace{-.1in}
\label{fig_matrix_comp}
\end{figure*}

\begin{figure*}[ht]
\centering
\begin{tabular}{@{}ccc@{}}
\includegraphics[width=5cm]{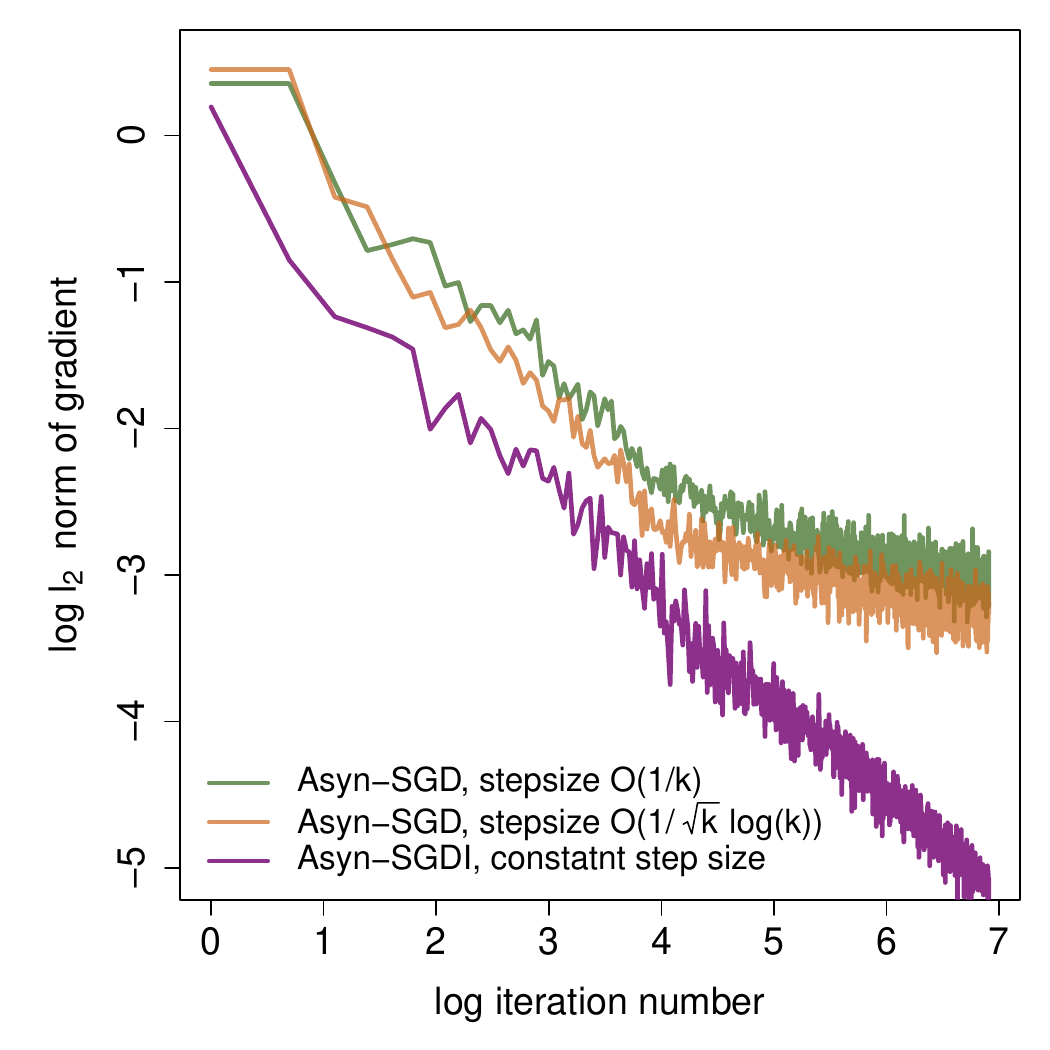}&
\includegraphics[width=5cm]{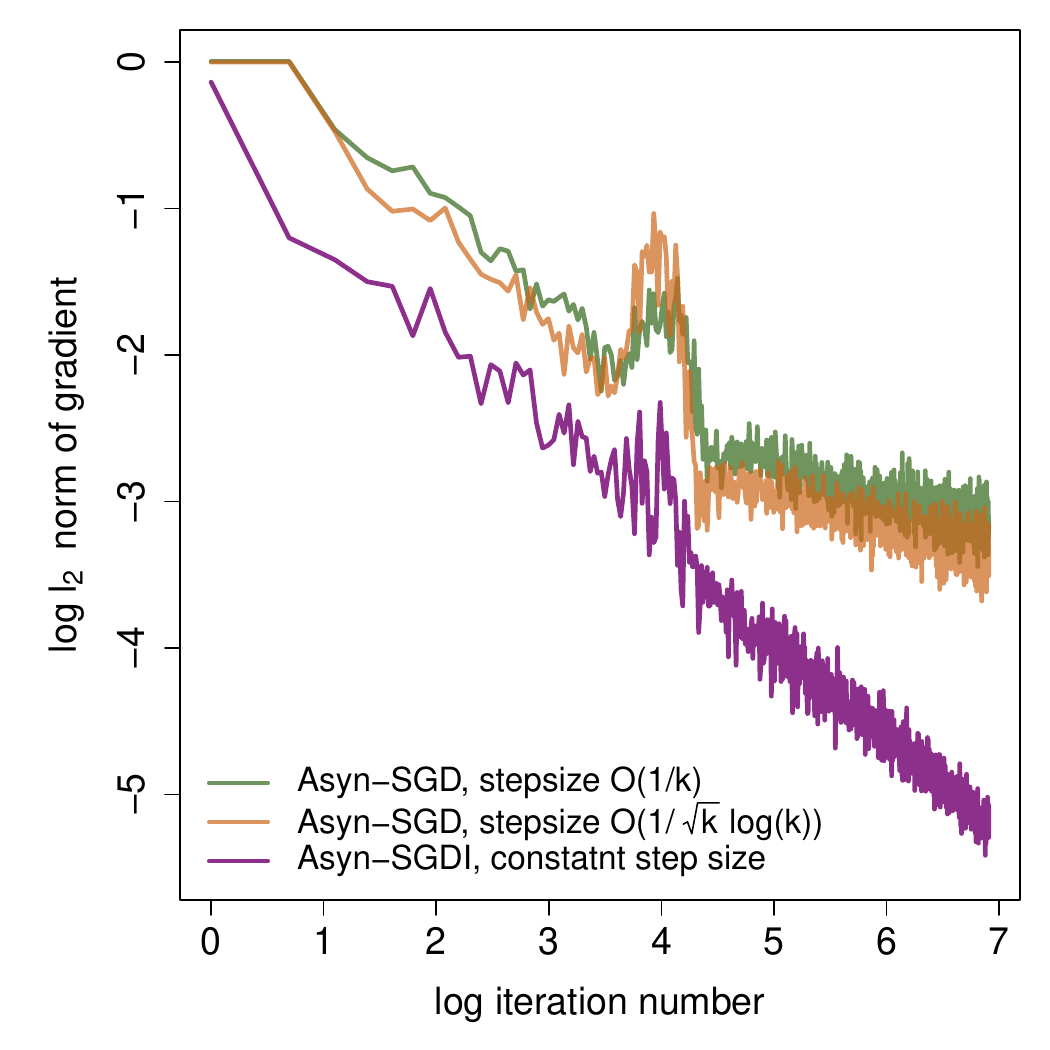}&
\includegraphics[width=5cm]{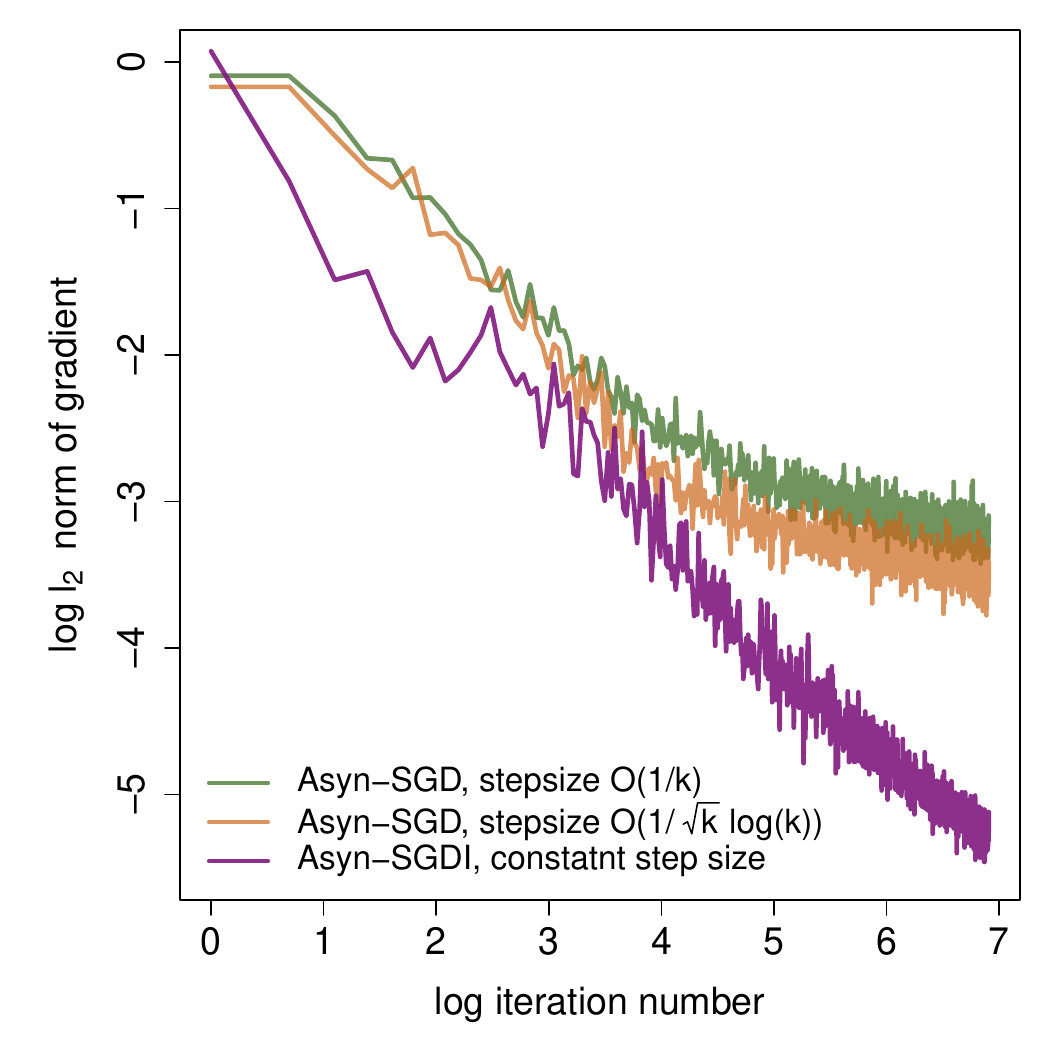}\\
\footnotesize (a) Bounded delay, $\tau\le 50$.&\footnotesize (b) Poisson delay, $\tau\sim Poi(30)$.& \footnotesize (c) System delay.
\end{tabular}\label{MVN}
\caption{Simulation results for the convergence of Async-SGD and Async-SGDI with three kinds of delay on MLE for MVN covariance matrix.} 
\vspace{-.1in}
\end{figure*} 

In this part, we will present several numerical experiments to further validate our theoretical results. 

\subsection{Low-rank matrix completion} \label{sec_low_rank}
First, we apply the Async-SGD algorithm in solving a low-rank matrix completion problem, where the goal of is to find the matrix $X$ with the lowest rank that matches the expectation of observed symmetric matrices, $\mathbb{E}\{A\}$. 
This problem could be mathematically formulated as follows:
\begin{align*}
\min_{Y\in \mathbb{R}^{n\times p}} & \mathbb{E}\{ \|A-YY^T\|^2_F \},
\end{align*} 
where $X=YY^T.$ 
Using SGD to solve this problem has been investigated in many works (see, e.g., \cite{de2014global,balzano2010online} etc.).

In our experiment, we consider three random delay scenarios: 
1) Delay is uniform at random with soupport being the interval [0,20];
2) i.i.d delay with poisson distribution, Poisson(10);
3) Non i.i.d delay, which we call system delay, is simulated from a virtual system with 10 workers whose computation time $t$ for a gradient follows a hierarchical distribution, $t\sim Exp(\lambda)$ and $\lambda \sim Gamma(2,1)$. We let the central server update the parameter when it collets $M$ gradients from the 10 workers.
Note here that for the third delay model, we consider the working time follows a Gamma-Exponential distribution, which are often used for modeling working time. 
The delay is caused by the difference between the working times. 
In addition, for the three scenarios, the delay is 0 in the first iteration.
And the delays of $M$ gradients in each iteration are different. 
But in each iteration, the delays are generated following the same distribution.

Aysnc-SGD and Aysnc-SGDI are applied on our simulated data: the ground truth is a randomly generated rank-one matrix $\mathbb{E}(A)$ and the observed samples are $\mathbb{E}(A)+\epsilon$, where the random variable $\epsilon$ is drawn from $N(0,1)$. 
For Async-SGD, we consider two sets of step-sizes. 
The first one is chosen as $\{1\times 10^{-6}, \frac{1}{2}\times 10^{-6},\frac{1}{3}\times 10^{-6},\ldots\}$, decaying every 10 iterations, which can be viewed as $O(1/k)$. The second one is $\{1\times 10^{-6}, \frac{1}{2\log(2)}\times 10^{-6},\frac{1}{3\log(3)}\times 10^{-6},\ldots\}$, also decaying every 10 iterations. 
It satisfies the $O(1/(k^{1/2}\log(k)))$ step-size bound in Proposition 1. 
We choose the batch size $M$ as $100$. 
For Async-SGDI, we choose fixed step-size as $10^{-6}$ and increase the batch size as $\{100,400,900,\ldots\}$ every $100$ iterations. 
We run both algorithms $5000$ iterations and illustrate the convergence behaviors of these two schemes with the $\ell_2$-norm of the gradients in Figure~\ref{fig_matrix_comp}.

In Figure~\ref{fig_matrix_comp}, both Async-SGD algorithms with three different types of random gradient delay variables are convergent. 
We can see that the Async-SGDI algorithm has the fastest convergence speed. Async-SGD with $O(1/(k^{1/2}\log(k)))$ step-size is faster than that with $O(1/k)$ step-size. 
This result is consistent with our theoretical analysis.

\subsection{Maximum likelihood estimation for multivariate normal covariance matrix}\label{sec:MLE}

The second problem we experimented is the maximum likelihood estimation for the covariance matrix of a multivariate normal distribution, which can be formulated as:
\begin{align*}
\underset{\Sigma\in \mathbb{R}^{d\times d}}\min & \ln{|\Sigma|}+\frac{1}{n}\sum_{i=1}^{n}(x_i-\mu)^T\Sigma^{-1}(x_i-\mu),
\end{align*} 
where $\Sigma$ is the covariance matrix to be estimated, $\mu$ is the mean vector and $x_i$ are the samples. 
The gradient for this problem has been derived in \cite{minka2000old}.
We randomly generate data from a multivariate normal distribuion with mean $\mu^T=(0,0,0,0,0)$ and covariance matrix $$\Sigma=\left[\begin{matrix}
  12.46 & 3.99 & 5.48 & 2.71 & 2.95 \\
  3.99 & 14.99 & 4.74 & 2.42 & 4.64 \\
  5.48 & 4.74 & 12.72 & 1.68 & 2.80 \\
  2.71 & 2.42 & 1.68  &16.15 &  3.82 \\
  2.95 & 4.64 & 2.80 &  3.82 &19.38 
  \end{matrix}\right].$$

Again, we apply Async-SGD and Async-SGDI on the simulated data with three different random gradient delay models as defined in Section~\ref{sec_low_rank}: a) bounded by 50; b) Poisson(30); and c) System delay. 
For Async-SGD, we choose batch size $M$ as $100$ and consider two sets of step-sizes. The first step-size is chosen as $\{1\times 10^{-3}, \frac{1}{2}\times 10^{-3},\frac{1}{3}\times 10^{-3},\ldots\}$, decaying every 50 iterations. And the second step-size is $\{1\times 10^{-3}, \frac{1}{2\log(2)}\times 10^{-3},\frac{1}{3\log(3)}\times 10^{-3},\ldots\}$, also decaying every 50 iterations.
For Async-SGDI, we choose step-size as $0.001$ and increase the batch size as $10k^2$, with every $100k$ iterations. 
We illustrate the convergence results of these two Async-SGD schemes with the $\ell_2$-norm of gradient in Figure 3.
From Figure 3, we can observe similar results. 
The $\ell_2$-norm of the gradients are decreasing as the number of iterations increases, regardless of the choice of random delay models. 
Among the three curves, the one for Async-SGDI converges the fastest and Async-SGD with $O(1/k)$ step-size is the slowest. 
These results confirms our theoretical analysis.


\section{Conclusion} \label{sec:conclusion}

In this paper, we analyzed the convergence of two asynchronous stochastic gradient descent methods, namely Async-SGD and Async-SGDI, for non-convex optimization problems. 
By constructing a Lyapunov function that combines optimality error and asynchronicity errors, we proved a convergence rate $o(1/\sqrt{k})$ for Async-SGD and a convergence rate $o(1/k)$ for Async-SGDI, respectively.
We note that both convergence results are stronger compared to previous work.
Also, we developed a generalized and more relaxed sufficient assumption on gradient update delay for Async-SGD's convergence.
This assumption provides a unifying framework that includes the major delay models in the existing works as special cases.  
Collectively, our results advance the understanding of the convergence performance of Async-SGD for non-convex learning with unbounded gradient update delay. 
Our future work may involve pursuing non-asymptotic convergence analysis with similar weak assumptions, as well as adaptive batch-size selection strategy for Async-SGDI to increase the computation speed.

\bibliographystyle{IEEEtran}
\bibliography{reference.bib}

\appendices




\section{Proofs for Lemma \ref{lemma: delay probability} } \label{appdx_lem1}

Lemma \ref{lemma: delay probability} (i.e., the existence of a non-negative sequence $\{c_i\}_{i=1}^{\infty}$) can be proved in a constructive fashion. 
Consider the sequence $\{c_i\}_{i=1}^{\infty}$ that satisfies:
\begin{equation}\label{cc3}
c_l\ge c_{l+1} + \frac{\gamma L^2}{2}\sum_{i=l}^{\infty}ia_i \ge c_{l+1} + \frac{\gamma_kL^2}{2}\sum_{i=l}^{\infty}i\mathbb{P}(\tau_k=i),
\end{equation} 
where $\gamma\ge\max \gamma_k$ is a constant and the second inequality follows from Assumption~\ref{assump_delay}.
If such a sequence $\{c_i\}_{i=1}^{\infty}$ exists, the proof is done.
To show that $\{c_{i}\}_{i=1}^{\infty}$ exists, we only need to prove that $c_1$ is finite. 
To this end, telescope the inequality in (\ref{cc3}), we have: 
\begin{align*}
c_1&=\frac{\gamma L^2}{2}\sum_{j=1}^{\infty}\sum_{i=j}^{\infty}ia_i=\frac{\gamma L^2}{2}\sum_{i=1}^{\infty} \sum_{j=1}^{i}ia_i\\
&=\frac{\gamma L^2}{2}\sum_{i=1}^{\infty} i^2 a_i < \infty
\end{align*}

Therefore, we could generate $\{c_i\}_{i=1}^{\infty}$ by the inequalities in (\ref{cc3}).
This completes the proof.

\section{Proofs of Lemma \ref{lemma1} and Theorem \ref{Theo1}} \label{appdx_thm1}

Lemma \ref{lemma1} can be proved as follows:
Define $\mathscr{F}^k = \sigma\langle x_1,\dots,x_k; \tau_1,\dots,\tau_k \rangle$, where $x_i$ is the $i$-th iterate, $\tau_i$ is the delay in $i$-th iteration.
First, according to the updating rule and Assumptions~\ref{assump_bnded_obj}--\ref{assump_grad_exp_var}, it holds that
\begin{align} \label{eqn_step1}
&
\mathbb{E}(f(x_{k+1})-f(x_k)|\mathscr{F}^k) \notag\\
 &
\stackrel{}{\le} -\frac{\gamma_kM}{2}\|\nabla f(x_k)\|^2 +\frac{\gamma_k^2LM\sigma^2}{2} \notag\\
&+(\frac{\gamma_k^2LM}{2}-\frac{\gamma_k}{2})\sum_{m=1}^{M}\mathbb{E}(\|\nabla f(x_{k-\tau_{k,m}})\|^2|\mathscr{F}^k) \nonumber\\
&
+\frac{\gamma_kL^2}{2}\sum_{m=1}^{M}\mathbb{E}(\|x_k-x_{k-\tau_{k,m}}\|^2|\mathscr{F}^k).
\end{align}
Note that the expectation of $\|x_k-x_{k-\tau_{k,m}}\|^2$ can be bounded as:
\begin{align*}
&
\mathbb{E}(\|x_k-x_{k-\tau_{k,m}}\|^2|\mathscr{F}^k) \notag\\
&
\stackrel{}{=} \sum_{i=1}^{k}\mathbb{P}(\tau_{k,m}=i)\|\sum_{j=1}^{i}x_{k+1-j}-x_{k-j}\|^2\\
&\le \sum_{i=1}^{k}\mathbb{P}(\tau_{k,m}=i)i\sum_{j=1}^{i}\|x_{k+1-j}-x_{k-j}\|^2 \\
&=  \sum_{j=1}^{k}\sum_{i=j}^{k}\mathbb{P}(\tau_{k,m}=i)i\|x_{k+1-j}-x_{k-j}\|^2,
\end{align*}
Consider the Lyapurov function:
\begin{equation}
\zeta^k=f(x_k)-f(x^*)+\sum_{j=1}^{\infty}c_j\|x_{k+1-j}-x_{k-j}\|^2,
\end{equation}
for which we have:
\begin{align*}
&
\mathbb{E}(\zeta^{k+1}|\mathscr{F}^k) \\
%
&\le f(x_k)-f(x^*)+\sum_{j=1}^{k}(c_{j+1}+\frac{\gamma_kL^2}{2}\sum_{m=1}^{M}(
\sum_{i=j}^{k}\\&i\mathbb{P}(\tau_{k,m}=i))\|x_{k+1-j}-x_{k-j}\|^2-\frac{\gamma_kM}{2}\|\nabla f(x_k)\|^2\\& + (c_1\gamma_k^2M+\frac{L\gamma_k^2M}{2})\sigma^2+(Mc_1\gamma_k^2+\frac{ML\gamma_k^2}{2}-\frac{\gamma_k}{2})\\&\sum_{m=1}^{M}\mathbb{E}(\|\nabla f(x_{k-\tau_{k,m}})\|^2|\mathscr{F}^k).
\end{align*}
With $\gamma_k \leq 1/(2Mc_{1}+ML)$, we have:
\begin{align*}
&
\mathbb{E}(\zeta^{k+1}|\mathscr{F}^k)+\frac{\gamma_kM}{2}||\nabla f(x_k)\|^2 \notag\\
&
\le f(x_k)-f(x^*)+(c_1\gamma_k^2M+\frac{L\gamma_k^2M}{2})\sigma^2 \\
&
+
\sum_{j=1}^{k}(c_{j+1}+\frac{\gamma_kML^2}{2}\sum_{i=j}^{k}i\mathbb{P}(\tau_{k}=i))\|x_{k+1-j}-x_{k-j}\|^2
\end{align*}
Next, using Assumption 4, we have:
\begin{align*}
\mathbb{E}(\zeta^{k+1}|\mathscr{F}^k)+\frac{\gamma_kM}{2}\|\nabla &f(x_k)\|^2\notag\\
&\le \zeta^{k}+ (c_1\gamma_k^2M+\frac{L\gamma_k^2M}{2})\sigma^2.
\end{align*}
This completes the proof of Lemma \ref{lemma1}.

To finish the proof of Theorem \ref{Theo1}, we now invoke with the Lemma 1 and use the following supermartingale convergence theorem, which has been used in \cite{hannah2016unbounded,combettes2015stochastic}:
\begin{myTheo}[\cite{hannah2016unbounded,combettes2015stochastic}] \label{thm_supermartingale}
Let $\alpha^k$, $\theta^k$ and $\eta^k$ be positive sequences adapted to $\mathscr{F}^k$, and let $\eta^k$ be summable with probably 1. If $$\mathbb{E}[\alpha^{k+1}|\mathscr{F}^k]+\theta^k\le \alpha^k+\eta^k,$$
then with probability $1$, $\alpha^k$ converages to a $[0,\infty)$-valued random variable, and $\sum_{k=1}^{\infty} \theta^k<\infty$.
\end{myTheo}

Applying above theorem with $\alpha^k=\zeta^k$, $\theta^k=\frac{\gamma_kM}{2}\|\nabla f(x_k)\|^2$ and $\eta^k=(c_1\gamma_k^2M+\frac{L\gamma_k^2M}{2})\sigma^2$, we have $\sum_{k=1}^{\infty}\frac{\gamma_kM}{2}\|\nabla f(x_k)\|^2 < \infty$ with probability $1$. Thus $\mathbb{E} \{ \sum_{k=1}^{\infty}\frac{\gamma_kM}{2}\|\nabla f(x_k)\|^2 \}< \infty$, which implies that $\mathbb{E}\{ \|\nabla f(x_k)\|^2 \} \rightarrow 0.$
This completes the proof.

\section{Proofs for Propositions \ref{prop_async_sgd}} \label{Append: proposition2}

Firstly, we have following facts. 
\begin{enumerate}
  \item for the $p$-series $\{\frac{1}{k^p}\}_{k=1}^{\infty}$ and the $\log$-series $\{\frac{1}{k\log^p(k)}\}_{k=2}^{\infty}$, they are both unsummable if $p\le 1$, while summable if $p>1;$    
    \item $\{\frac{1}{\sqrt{k}\log(k)}\}_{k=2}^{\infty}=\Omega(1/k)$. Hence$\{\frac{1}{\sqrt{k}\log(k)}\}_{k=2}^{\infty}$ is unsummable;
    \item the log-series $\{\frac{1}{k\log(k)}\}_{k=2}^{\infty}$ ($p=1$) is unsummable, while $\{\frac{1}{k\log^2(k)}\}_{k=2}^{\infty}$ is  summable ($p=2$);
\end{enumerate}
As a result, $\gamma_k=O(\frac{1}{k^{1/2}\log(k)})$ is unsummable and $\gamma_k^2=O(\frac{1}{k\log^2(k)})$ is summable (from fact 1 and 2).  
These are condition ii) and iii) in Theorem \ref{Theo1}.
It then follows from Theorem \ref{Theo1} that $\mathbb{E}[\sum_{k=1}^{\infty}\gamma_k \|\nabla f(x_k)\|^2]< \infty.$

To show the $o(1/\sqrt{k})$ convergence rate, we note that
if $\mathbb{E}\{\|\nabla f(x_k)\|^2\}=O(1/\sqrt{k})$, then $\mathbb{E}\{\gamma_k\|\nabla f(x_k)\|^2\}=O(\frac{1}{k\log(k)})$, which is unsummable (from fact 3) and contradict to our earlier conclusion that $\mathbb{E}\{\sum_{k=1}^{\infty}\gamma_k\|\nabla f(x_k)\|^2\}< \infty$. 
Therefore, $\mathbb{E}\{\|\nabla f(x_k)\|^2\}$ must be $o(1/\sqrt{k}).$


\end{document}